\documentclass[letterpaper, 10 pt, conference]{ieeeconf}  

\IEEEoverridecommandlockouts                              

\usepackage{amsmath}

\usepackage{amsthm}
\usepackage[font=footnotesize,labelfont=normalfont]{caption}
\usepackage{subcaption}
\usepackage{xcolor}
\usepackage{hyperref}
\usepackage{algpseudocode}
\usepackage[linesnumbered,ruled,vlined]{algorithm2e}
\usepackage{amsfonts}
\usepackage{svg}
\usepackage{cite}

\DeclareMathOperator*{\argmin}{arg\,min}
\newtheorem{theorem}{Theorem}
\usepackage{changepage}
\usepackage{siunitx}
\usepackage{tikz}
\usetikzlibrary{fit,calc}
\newcommand*{\tikzmk}[1]{\tikz[remember picture,overlay,] \node (#1) {};\ignorespaces}
\newcommand{\boxit}[1]{\tikz[remember picture,overlay]{\node[yshift=3pt,fill=#1,opacity=.25,fit={(A)($(B)+(.80\linewidth,.8\baselineskip)$)}] {};}\ignorespaces}
\colorlet{mypink}{red!40}
\colorlet{myyellow}{yellow!70}
\colorlet{myblue}{cyan!60}

\title{\LARGE \bf
Hierarchical Large Scale Multirobot Path (Re)Planning
}

\author{Lishuo Pan, Kevin Hsu, and Nora Ayanian
\thanks{This work was supported by NSF grants 1837779, 2317145, 2311967, 2330942. All authors are with the Department of Computer Science, Brown University, Providence, RI 02912, USA.
        Email: {\tt\footnotesize \{lishuo\_pan, kevin\_hsu, nora\_ayanian\}@brown.edu}}%
}

\begin{document}

\maketitle
\thispagestyle{empty}
\pagestyle{empty}

\begin{abstract}
We consider a large-scale multi-robot path planning problem in a cluttered environment. Our approach achieves real-time replanning by dividing the workspace into cells and utilizing a hierarchical planner. Specifically, we propose novel multi-commodity flow-based high-level planners that route robots through cells with reduced congestion, along with an anytime low-level planner that computes collision-free paths for robots within each cell in parallel. A highlight of our method is a significant improvement in computation time. Specifically, we show empirical results of a 500-times speedup in computation time compared to the baseline multi-agent pathfinding approach on the environments we study. We account for the robot's embodiment and support non-stop execution with continuous replanning. We demonstrate the real-time performance of our algorithm with up to 142 robots in simulation, and a representative 32 physical Crazyflie nano-quadrotor experiment.
\end{abstract}

\section{Introduction}
Large fleets of robots, 
such as those used in warehouse operations~\cite{wawrla2019applications}, disaster response~\cite{van2020drones}, and delivery~\cite{Scott2017DroneDM}, demand 
coordination solutions that adjust in real time to changing goals.  
In this work, we present a real-time lifelong hierarchical method for navigating a large team of robots to 
independent goals in a large, cluttered environment that 
guarantees collision avoidance. 
By lifelong, we mean robots can enter and exit the space, and can receive another goal at any time, as they would in a warehouse or delivery problem.
Our approach 
partitions the space into disjoint cells, allowing planning algorithms to run concurrently in parallel within each cell. 
A novel high-level planner routes robots through the partition, while a low-level anytime multi-agent pathfinding (MAPF) algorithm navigates robots to 
local goals within each cell in parallel. 
The real-time property holds as long as there are not too many cells or robots in a workspace; the limits for real-time operation are empirical and problem-specific, however, we demonstrate real-time performance for 142 robots in simulation with a 25-cell partition.

\begin{figure}[t]
    \centering
    \includegraphics[width=0.48\textwidth]{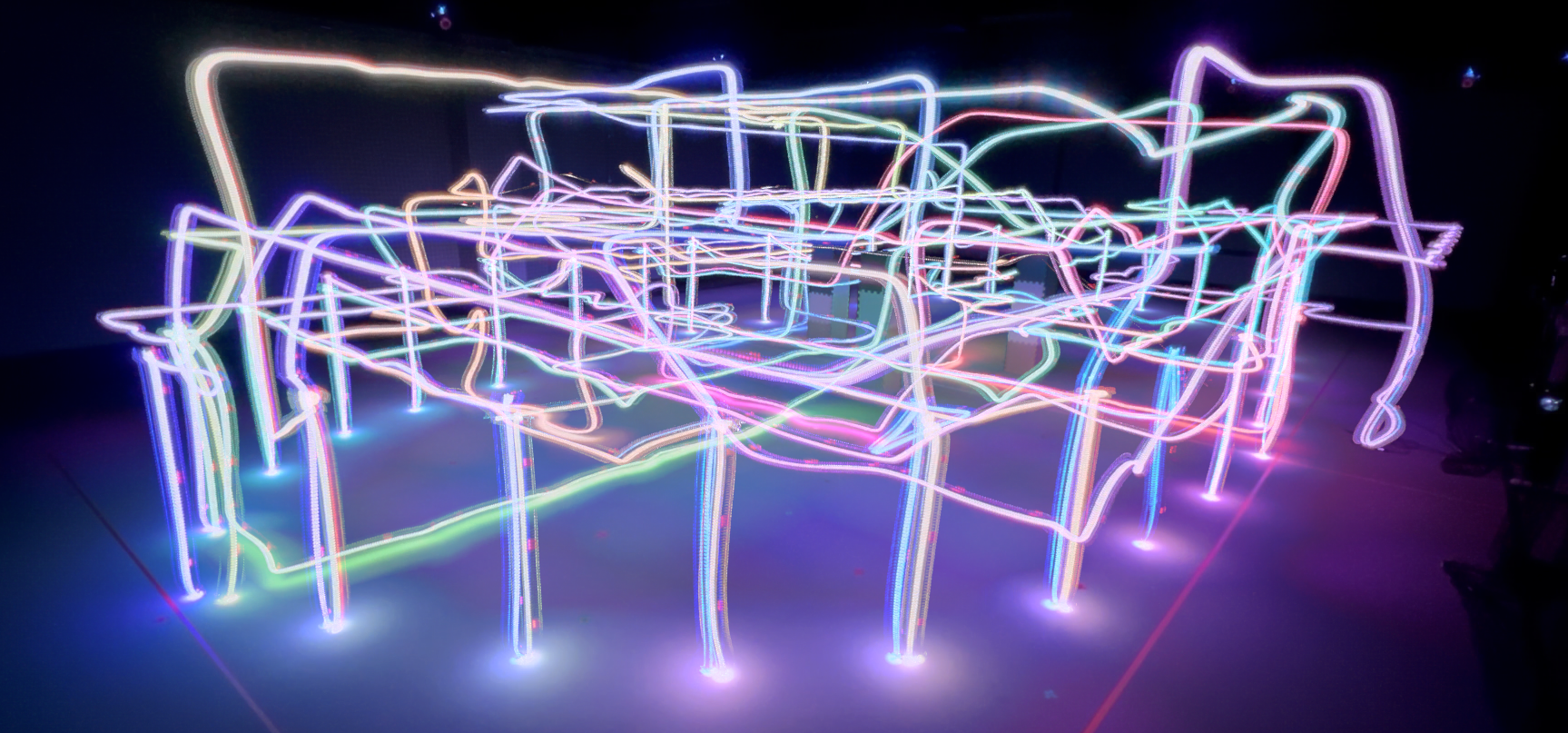}
    \caption{Long exposure of 32 quadrotors navigating a cluttered environment.}
    \label{fig:demo}
    \vspace{-2em}
\end{figure}

We are primarily interested in unmanned aerial vehicles (UAVs) operating in 3D space, such as in city-scale on-demand UAV package delivery; however, our approach applies to robots operating in 2D as well.
We present two approaches for high-level planning depending on the problem's requirements: 1) an egocentric greedy approach that always operates in real-time and 2) a novel high-level planner that routes robots through the partition using multi-commodity flow (MCF)~\cite{10.5555/137406}. 
There are tradeoffs between these two approaches.
The egocentric greedy planner operates in real-time regardless of the number of cells; however, it has no mechanism for distributing robots, thus it can 
result in congestion and longer low-level planning times within some cells. 
On the other hand, the MCF-based approach eases cell congestion by regulating the flow of robots into each cell while ensuring bounded-suboptimal inter-cell routing; thus, it can be useful in environments
such as urban UAV package delivery, where different types of cells (e.g., residential vs.\ highway) may have different limits on the influx of robots. The MCF-based planners can operate in real-time under certain conditions, thus allowing for lifelong replanning while reducing congestion, which  leads to faster,  real-time low-level planning 
within each cell. 


The low-level planner prohibits collisions while respecting the robots' geometric shapes. A novel cell-crossing protocol allows robots to transition between cells without stopping in midair. 
Combined with the MCF-based planner, this allows real-time computation and safe, non-stop execution of multi-robot plans.
The contributions of this work are:
\vspace{-2pt}
\begin{itemize}
    \item  a hierarchical framework for large-scale multi-robot real-time coordination that significantly reduces computation time compared to the baseline MAPF solver, while resulting in a moderately suboptimal solution; and
    \item  novel multi-commodity flow-based high-level planners, MCF/OD and one-shot MCF, that reduce congestion by regulating the influx of robots to each cell.%
    \vspace{-2pt}%
\end{itemize}
We demonstrate the algorithm in simulation with up to 142 robots and in physical robot experiments with 32 nano-quadrotors in cluttered environments, shown in Fig.~\ref{fig:demo}.

\section{Related Work}
Centralized approaches to multi-robot planning \cite{honig2018trajectory, yu2018effective} face substantial computational challenges due to their theoretical hardness~\cite{yu2013structure}, 
 prohibiting real-time replanning for many-robot systems. 
Decentralized online approaches, on the other hand, can result in deadlocks, livelocks, congestion, collision, and reduced efficiency. 
RLSS~\cite{csenbacslar2023rlss} generates trajectories based on a single-robot planner without a deadlock-free guarantee. Furthermore, control barrier functions~\cite{wang2017safety}, or reactive control synthesis methods, are prone to deadlocks. Distributed MPC~\cite{luis2020online} results in deadlocks when narrow corridors are present. 
In cluttered environments, a buffered Voronoi cell-based algorithm~\cite{zhou2017fast} also suffers from potential deadlocks, and an algorithm using relative safe flight corridor~\cite{park2020online} leads to collisions.
In the present work, we aim to address these problems and facilitate real-time MAPF at the discrete planning phase. 


Search-based MAPF solvers generate deadlock- and collision-free paths, but the complexity of optimal solutions scale exponentially with the number of agents~\cite{yu2013structure}. 
Bounded-suboptimal algorithms have been proposed to overcome this complexity, but poor scalability still prevents their application to real-time coordination for large teams. To address this, partition-based MAPF~\cite{zhang2021hierarchical, leet2022shard} divides the workspace into cells, reducing the complexity in each cell. However, inter-cell routing is either single-agent based~\cite{zhang2021hierarchical} or the solution of a multi-commodity flow problem constrained on single-robot shortest paths~\cite{leet2022shard} without congestion-awareness, leading to hard instances within regions. 
Instead, we propose a novel congestion-aware inter-cell routing algorithm to distribute robots while maintaining bounded-suboptimality. Furthermore, the cell-crossing method in~\cite{leet2022shard} is unsuitable for aerial vehicles due to energy expenditure while stationed in hover for the cell-crossing channel to clear. Our novel cell-crossing protocol addresses this drawback. 
\begin{figure}[t]
    \centering
    \includegraphics[width=0.45\textwidth]{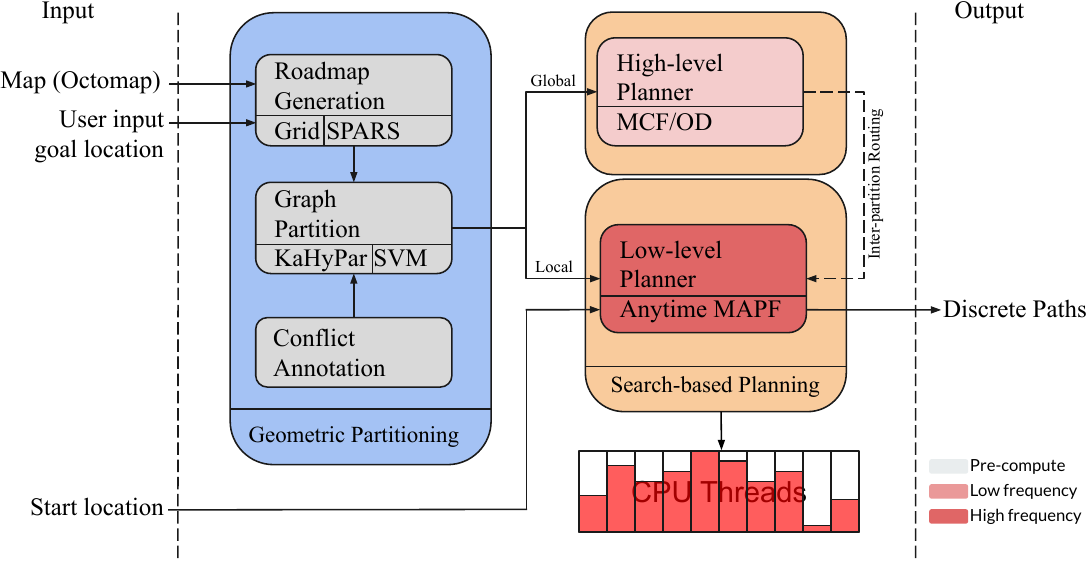}
    \caption{The user inputs a map, start, and goal locations. Our approach generates a geometric partition, distributes robots among  cells (hierarchical planner), and coordinates them within each cell in parallel.\label{fig:outline}}
    \vspace{-2em}
\end{figure}
\section{Preliminaries}
\label{sec:preliminaries}
\subsection{Multi-agent Pathfinding (MAPF) on Euclidean Graph}
\label{sec:preliminaries_mapf}
Consider an undirected graph $\mathcal{G}\!\!=\!\!\left(V, E\right)$ embedded in a Euclidean space, where each vertex $v\!\!\in\!\! V$ corresponds to a position in the free space $\mathcal{F}$ and each edge $(u, v) \!\!\in\!\! E$ denotes a path in $\mathcal{F}$ connecting vertices $u$ and $v$. For $N$ agents, we additionally require the existence of vertices $v^{i}_{g}$ and $v^{i}_{s}$, corresponding to the goal and start positions, $\mathbf{g}^{i}$ and $\mathbf{s}^{i}\!\! \in \!\! \mathbb{R}^{3}$ of robot $r^{i}$ (superscript $i$ represents  robot index), respectfully. At each time step $k$, an agent can either move to a neighbor vertex $\left(u^{i}_{k}, u^{i}_{k+1}\right)\!\!\in\!\! E$ or stay at its current vertex $u^{i}_{k+1}\!\! =\!\! u^{i}_{k}$, where $u^{i}_{k} \!\!\in\!\! V$ is the occupied vertex in the $i$-th agent's path at time step $k$. 
To respect vertex conflict constraints,  no two agents can occupy the same vertex simultaneously, i.e., $\forall k, i\!\!\neq\!\! j \!\!:\!\! u^{i}_{k} \neq u^{j}_{k}$. To respect edge conflict constraints, no two agents can traverse the same edge in the opposite direction concurrently, i.e., $\forall k, i\!\!\neq\!\! j\!\!:\!\! u^{i}_{k} \neq u^{j}_{k+1} \!\vee \!u^{i}_{k+1} \!\!\neq\!\! u^{j}_{k}$. 
The objective is to find conflict-free paths 
$\mathcal{P}^{i}\!\!=\!\!\left[u^{i}_{0}, \cdots, u^{i}_{T-1}\right]$, where $u^{i}_{0} \!\!=\!\! v^{i}_{s}$ and $u^{i}_{T-1} \!\!=\!\! v^{i}_{g}$ for all agents, and minimize cost, e.g., the sum over the time steps required to reach the goals of all agents or the makespan $T$. 
\subsection{MAPF for Embodied Agents and Conflict Annotation}
\label{sec:MAPFC}
Many works address MAPF for embodied agents~\cite{li2019multi, walker2018extended, yakovlev2017any}.
We adopt multi-agent pathfinding with generalized conflicts (MAPF/C), due to its flexibility with different shapes. Generalized conflicts, different from typical MAPF conflicts, include the extra conflicts caused by the robot embodiment~\cite{honig2018trajectory}. 
To account for the downwash effect between robots~\cite{yeo2015empirical}, we denote $\mathcal{R}_{\mathcal{R}}(\mathbf{p})$ as the convex set of points representing a robot at position $\mathbf{p}$, i.e., a robot-robot collision model. We follow the conflict annotation in~\cite{honig2018trajectory} to annotate the graph with generalized conflicts, defined as: 
\vspace{-0.5em}
\begin{align}\nonumber
\operatorname{conVV}(v)= & \left\{u \in V \mid\right. \\\nonumber
& \left.\mathcal{R}_{\mathcal{R}}(\operatorname{pos}(u)) \cap \mathcal{R}_{\mathcal{R}}(\operatorname{pos}(v)) \neq \emptyset\right\} \\\nonumber
\operatorname{conEE}(e)= & \left\{d \in E \mid \mathcal{R}_{\mathcal{R}}^{*}(d) \cap 
\mathcal{R}_{\mathcal{R}}^{*}(e) \neq \emptyset\right\}\\\nonumber
\operatorname{conEV}(e)= & \left\{u \in V \mid\right.
\left.\mathcal{R}_{\mathcal{R}}(\operatorname{pos}(u)) \cap \mathcal{R}_{\mathcal{R}}^{*}(e) \neq \emptyset\right\},
\end{align}
where $\operatorname{pos}(u) \!\! \in \!\! \mathbb{R}^{3}$ returns the position of vertex $u$. $\mathcal{R}_{\mathcal{R}}^{*}(e)$ is the $\textit{swept}$ collision model representing the set of points swept by the robot when traversing edge $e$. 

\section{Problem Formulation}
Consider a time-varying number of homogeneous non-point robots in workspace $\mathcal W$, 
which is partitioned into a union of disjoint convex polytopic cells. Robots must reach specified individual goal positions, which change over time, while avoiding collisions with robots and obstacles and obeying maximum cell influx limits $\boldsymbol{\theta}$ (influx refers to the number of robots entering a workspace cell).
A motivating scenario is a multi-UAV package delivery system, where the number of UAVs entering certain types of airspaces must be limited, and UAVs exit or enter the workspace to charge or redeploy. 
The problem above requires solving a path replanning problem that obeys  cell influx limits while handling new goal positions and a varying number of robots. While that is the general problem we aim to solve, the present work addresses a critical subproblem: the underlying planner that is repeatedly called to safely and efficiently route the robots through $\mathcal W$.

Now, consider $N$ (fixed) homogeneous non-point robots operating in the partitioned workspace. $\mathcal W$ contains obstacles defined as unions of convex polytopes $\mathcal{O}_{1}, \cdots, \mathcal{O}_{N_{obs}}$. We use a robot-environment collision model $\mathcal{R}_{\mathcal{E}}(\mathbf{p})$ that is distinct from $\mathcal{R}_{\mathcal{R}}(\mathbf{p})$.
The free space is $\mathcal{F} \!\! = \!\! \mathcal{W} \backslash (\bigcup_{h} \mathcal{O}_{h})  \ominus \mathcal{R}_{\mathcal{E}}(\mathbf{0})$, where  $\ominus$ is Minkowski difference.
For each robot $r^{i}$ at initial position $\mathbf{s}^{i}$, our algorithm finds a path to its goal position $\mathbf{g}^{i}$ such that there are no collisions (e.g., between robots or between robots and $\bigcup_{h}\mathcal{O}_h$) and the total number of robots that enter each cell is less than its user-defined influx $\theta_{m}$ (influx limits can vary by cell). 

To efficiently address the above problem, our framework has three components: geometric partitioning, high-level planner, and low-level planner, as depicted in Fig.~\ref{fig:outline}. Geometric partitioning divides the workspace into disjoint convex cells, where the plan can be computed in parallel. A centralized high-level planner regulates the congestion for each cell while guaranteeing inter-cell routing quality. An anytime low-level planner plans collision-free paths for robots within each cell. Initial planning includes pre-computation (the geometric partitioning), and replanning only involves high- and low-level planning. Once the initial plan is established, our algorithm runs in real-time for replanning.

\section{Geometric Partitioning}
The geometric partitioning of a bounded workspace 
consists of three steps: 1) roadmap generation, 2) graph partitioning and spatial linear separation, and 3) local goal generation. 
\subsection{Roadmap Generation}
A roadmap is an undirected graph, introduced in Sec.\ref{sec:preliminaries_mapf}, 
satisfying three properties: 1)  connectivity-preserving,
i.e., if a path between two points in $\mathcal{F}$ exists, there should be a path in the roadmap as well; 
2) optimality-preserving, i.e., the shortest path between two points in $\mathcal{F}$ can be well approximated by a path in the roadmap; and 3) sparse, i.e., have a small number of vertices and edges. 
In our experiments, we use a 6-connected grid graph, however, it can be generated by other methods, such as SPARS~\cite{dobson2014sparse}. 
\begin{figure}[t]
    \begin{subfigure}{0.115\textwidth}
       \centering {\footnotesize(a)}\includegraphics[height=.95\textwidth]{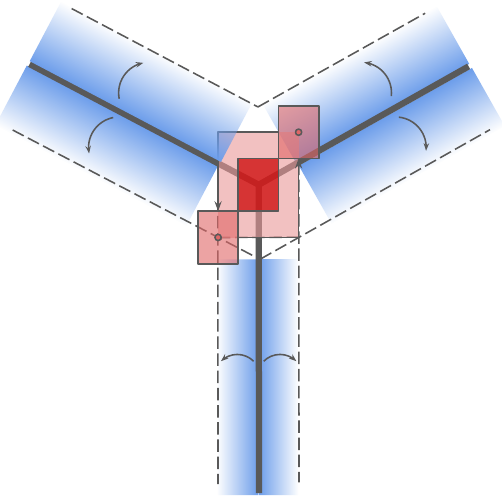}
    \label{fig:partition_generalized_conflict_1}
    \end{subfigure}
    \begin{subfigure}{0.115\textwidth}
        \centering {\footnotesize(b)}\includegraphics[height=.95\textwidth]{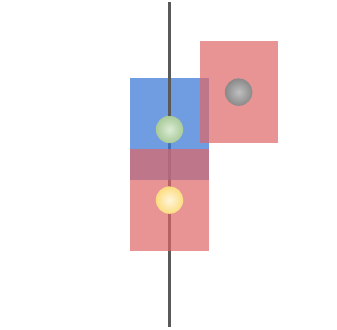}
    \label{fig:partition_generalized_conflict_2}
    \end{subfigure}
    \begin{subfigure}{0.115\textwidth}
        \centering {\footnotesize(c)}\includegraphics[height=.95\textwidth]{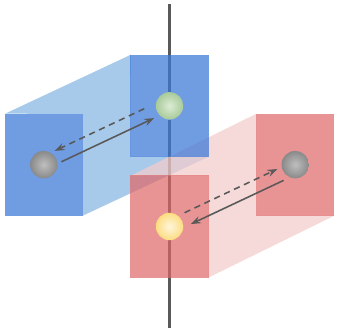}
    \label{fig:partition_generalized_conflict_3}
    \end{subfigure}
    \begin{subfigure}{0.115\textwidth}
        \centering {\footnotesize(d)}\includegraphics[height=.95\textwidth]{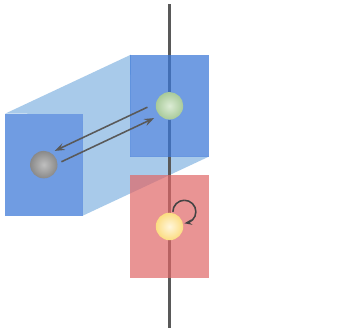}
    \label{fig:partition_generalized_conflict_4}
    \end{subfigure}
    \vspace{-1em}
\caption{(a) inter-robot collision configuration $\mathcal{C}_{col}$ and buffered hyperplanes. (b)-(d) vertex-vertex, edge-edge, and edge-vertex generalized conflicts across cells.\label{fig:partition_generalized_conflict}} 
\vspace{-2.0em}
\end{figure}
\subsection{Graph Partitioning and Spatial Linear Separation}
Our method partitions the workspace into disjoint convex cells and solves a MAPF instance in each cell. Partitioning has two benefits: 1) fewer robots for each subgraph, and 2) decomposed instances can be solved in parallel.
While the user can adopt any generic workspace partitioning approach, 
we propose a method that generates $Q$ convex polytopes. 
First, we use graph partitioning (KaHyPar~\cite{gottesburen2020advanced}) 
to group the roadmap into $Q$ \textit{balanced} (similar vertex numbers) subgraphs $\mathcal{G}_{m}\!\!=\!\!\left({V_{m}, E_{m}}\right)$, for $m\!\!=\!\!1,\ldots,Q$. 
\textit{Balanced} subgraphs lead to cells of similar volume and an even spread of robots.

We further enforce each cell, containing a subgraph, to be a convex polytope. Cell convexity prevents robots from penetrating into neighboring cells before exiting their current cells.
We use soft-margin support vector machines (SVM)~\cite{cortes1995support} to compute a hyperplane $H_{ml}$ between vertices of $\mathcal{G}_{m}$ and  $\mathcal{G}_{l}$, and reassign misclassified vertices. Here, the subscript $l$ denotes the index of the subgraphs with vertices connecting to vertices of $\mathcal{G}_{m}$. The resulting set of hyperplanes forms the $m$-th cell, denoted as $P_{m}$. 
\subsection{Local Goal Generation}
For navigating out of a cell, we generate candidate goal states on the faces between adjacent cells. For each cell $P_{m}$, we uniformly sample random local goals on the hyperplane $H_{ml}$ and add them as shared vertices to both $\mathcal{G}_{m}$ and $\mathcal{G}_{l}$. Despite being shared vertices, local goals generated by partition $P_{m}$ have in-edges from $P_{m}$ and out-edges to $P_{l}$ to avoid collision during cell transit. Thus, this part of the graph is directed. 
To enable parallel computation,  there must be no communication between cells. Thus, the cell roadmap $\mathcal{G}_{m}$ is modified such that the planned paths are collision-free when crossing cells without information exchange between cells. The following properties should be satisfied:

\noindent\textbf{P1}: Robots avoid collision when stationary at vertices (local-goal or non-local-goal vertices) of different cells (generalized vertex-vertex conflict across cells), i.e., $\forall i, j, m\!\!\neq\!\! l\!:\! v^{i}_{m} \!\!\notin\!\! \operatorname{conVV}(v^{j}_{l})$, where the superscript in $v^{i}_{m}$ refers to the vertex index and the subscript refers to the cell index.


\noindent\textbf{P2}: Robots avoid collision when traversing edges 
between different cells (generalized edge-edge conflict across cells), i.e., $\forall i, k, m\!\!\neq\!\! l: e^{i}_{m}\!\!:=\!\!(v^{i}_{m}, v^{j}_{m}) \notin \operatorname{conEE}(e^{k}_{l})$, here the superscript in $e^{i}_{m}$ refers to the edge index. 

\noindent\textbf{P3}: Robots avoid collision when one robot is stationary at a vertex while the other robot is traversing an edge of a different cell (generalized edge-vertex conflict across cells), i.e., $\forall i, k, m\!\!\neq\!\! l: e^{i}_{m}:=(v^{i}_{m}, v^{j}_{m}), v^{k}_{l} \notin \operatorname{conEV}(e^{i}_{m})$.

We depict violations of P1-P3 in Fig.~\ref{fig:partition_generalized_conflict}b-\ref{fig:partition_generalized_conflict}d. To prevent conflicts between stationary robots at local-goal and non-local-goal vertices across cells, we buffer the separating hyperplane by the inter-robot collision configuration, $\mathcal{C}_{col}$ (c.f. Fig.~\ref{fig:partition_generalized_conflict}a), which is computed as $\mathcal{C}_{col}(\mathbf{p}) = \mathcal{R}_{\mathcal{R}}(\mathbf{p}) \oplus \mathcal{R}_{\mathcal{R}}(\mathbf{0})$, where $\oplus$ is the Minkowski sum. The buffering is achieved by modifying the offset $\mathcal{H}_{a}^{'} = \mathcal{H}_{a} + \mathrm{max}_{\mathbf{y}\in \mathcal{C}_{col}(\mathbf{0})} \mathcal{H}_{\mathbf{n}} \!\cdot\! \mathbf{y}$ of the hyperplane, where $\mathcal{H}_{a}$ and $\mathcal{H}_{\mathbf{n}}$ are the offset and the normal vector of the hyperplane. Given the buffered hyperplanes, all non-local-goal vertices within the buffered region are removed, preventing collisions between stationary robots at local-goal and non-local-goal vertices. The local goals sampled on the hyperplane are confined within the blue region in Fig.~\ref{fig:partition_generalized_conflict}a to avoid vertex-vertex conflicts between stationary robots at local goals of different separating hyperplanes. 
The buffering satisfies P1 between local-goal and non-local-goal vertices across cells and P1, P2, and P3 between non-local-goal vertices across cells. 
To satisfy P1 between local-goal vertices across cells, we uniformly randomly sample points on the hyperplane and reject those that violate P1. 
For P2 and P3 between local-goal and non-local-goal vertices, we add the connection between the sampled local-goal vertex to valid (no collision with the environment) non-local-goal vertices within a radius on both sides of the hyperplane. The local goal is removed if any connected edge violates P2 or P3 (see Fig.~\ref{fig:partition_generalized_conflict}c,~\ref{fig:partition_generalized_conflict}d). Otherwise, the edges are added to the cell roadmaps. Figure~\ref{fig:geometric_partition}  depicts an exemplar partition. 
\begin{figure}[t]
    \hspace{3em}
    \begin{subfigure}{0.13\textwidth}
       \centering {\footnotesize(a)}\includegraphics[height=.95\textwidth]{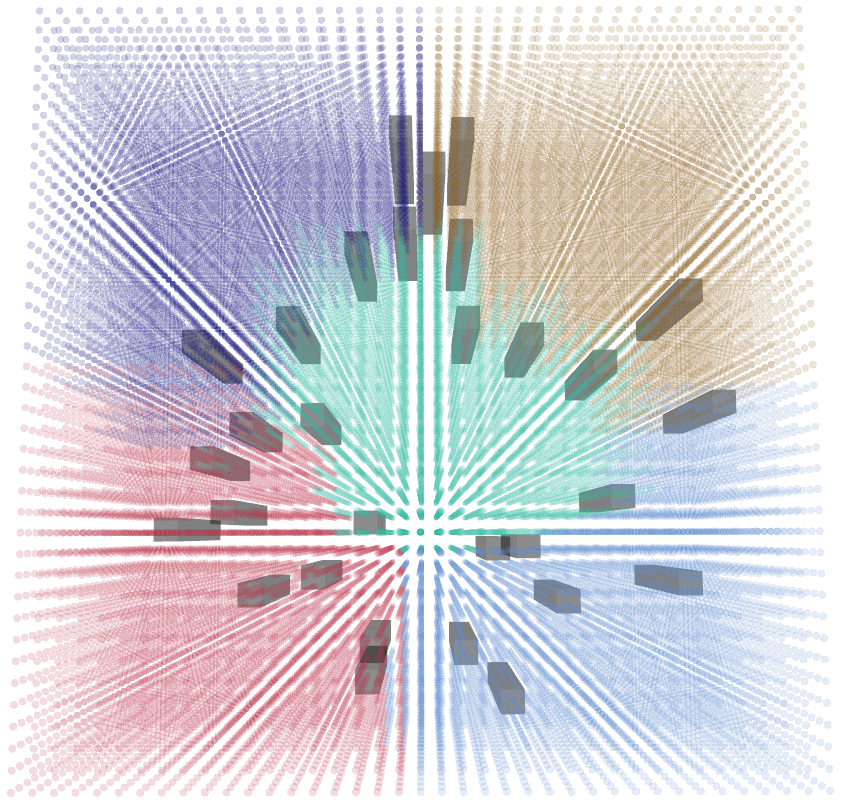}
    \end{subfigure}
    \hspace{2em}
    \begin{subfigure}{0.13\textwidth}
        \centering {\footnotesize(b)}\includegraphics[height=.95\textwidth]{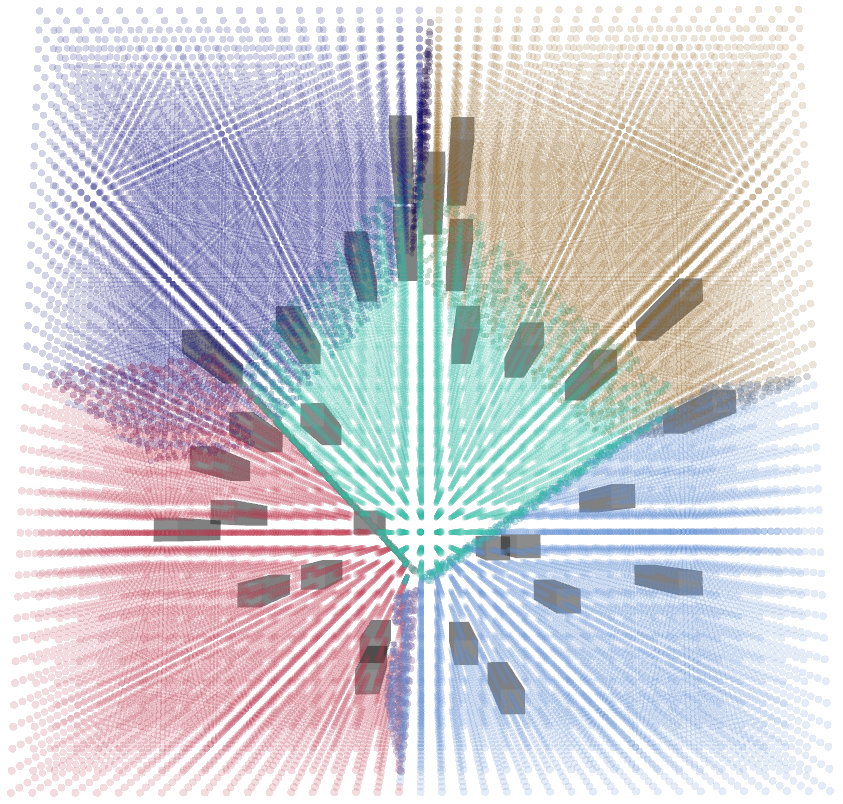}
    \end{subfigure}
\caption{Geometric partitioning with $Q\!\!=\!\!5$, before spatial linear separation (a), and after (b). Each cell is a convex polytope in the workspace.\label{fig:geometric_partition}} 
\vspace{-1.5em}
\end{figure}

\section{Multi-commodity Flow with Optimal Detour}
\label{sec:MCF}
Our hierarchical approach relies on high-level planning to 1) regulate cell congestion and 2) preserve the bounded-suboptimality of inter-cell routing solutions. 
Thus, our high-level planner simplifies  MAPF instances within cells and leads to real-time replanning.
We abstract the partition as a directed graph $\mathcal{G}_{p}\!\!=\!\!\left(V_{p}, E_{p}\right)$, where the vertices (nodes) represent cells and edges connect neighboring cells that share at least one face. Edges are weighted according to Euclidean distance between the cells' centers of mass.
The high-level planner finds an inter-cell routing $\mathcal{U}^{i}\!\!=\!\! \left[ P^{i}_{s}, \cdots, P^{i}_{g}\right]$ for each robot $r^{i}$, where $P^{i}_{s}$ and $P^{i}_{g}$ are its start and goal cell, satisfying: 1) the influx of cell $m$ is under a user-defined value $\theta_{m}$, and 2) $cost(\mathcal{U}^{i})\!\!\leq\!\! w_{mcf}\cdot cost(\mathcal{U}^{i,*})$. Here, $w_{mcf} \!\!\geq\!\! 1$ is a scalar representing the suboptimality bound for the routing solutions. $\mathcal{U}^{i,*}$ is the optimal routing of robot $r^{i}$. 
The influx of a node represents the number of robots entering the cell; we define influx formally in the following section.
\subsection{High-level Planning Formulation}
\label{sec:high_level_formulation}
The SOTA partition-based MAPF solvers~\cite{zhang2021hierarchical, leet2022shard} suffer from cell congestion,  leading to hard instances in certain cells, causing computational bottlenecks. To address this, we formulate the inter-cell routing as a variant of the MCF problem and propose multi-commodity flow with optimal detour (MCF/OD), which optimally distributes robots among cells such that the number of robots entering any intermediate cell $m$ (not the start or goal cells of the robot teams) is under a user-defined value $\theta_{m}$, if a solution exists.

Specifically, robots sharing the same start cell $P_{s}$ and goal cell $P_{g}$ are one commodity $c_{sg}$. The commodities set $C\!\!=\!\!\left\{c_{1}, \cdots, c_{O}\right\}$ includes all commodities given the robot positions. Solving the MCF problem results in optimal flows $\{ y^{*}_{sgml}\}$, that is the number of robots in commodity $c_{sg}\!\!\in\!\! C$ traversing along edge $e_{ml}\!\in\! E_{p}$.    
In our minimal influx MCF formulation, the optimal flow solutions lead to minimized intermediate cell influx and the most dispersed routing. 
We define the cell influx of $P_{l}$ as the total number of entering robots, that is, $\sum_{c_{sg}\in C, e_{ml}\in E_{p}} y_{sgml}$.
We formulate the following integer linear program (ILP):
\begin{subequations}
\begin{IEEEeqnarray}{rCl'rCl'rCl}
\argmin_{  \left\{y_{sgml}\right\}} \quad \alpha \!&~&\cdot \! \sum_{c_{sg}}\mathcal{L}_{sg} + \beta \! \cdot \! \mathcal{L}_{in}   \\
\text{s.t.}\sum_{e^{ml},e^{ln}\in E_{p}} \!\!\!\!&&y_{sgml} - y_{sgln} = 
    \begin{cases} 
    |c_{sg}|,  &l = g\\
    -|c_{sg}|,  \!\!\!\!&l=s\\
    0,  &o.w.
    \end{cases}, \forall c_{sg} \in C \label{constr:mcf}\\
\mathcal{L}_{sg} &&\geq y_{sgml},~\forall e_{ml}\in E_{p}, \forall c_{sg} \in C \label{constr:L_sg}\\
\mathcal{L}_{in} &&\geq \mathbf{I}^{\top}_{v^{i}}\mathbf{y},~\forall v^{i} \in V_{p} \label{constr:L_in}\\
y_{sgml} &&\in \left[0, |c_{sg}|\right],~\forall e_{ml}\in SP_{sg}, \forall c_{sg} \in C\label{constr:SP_constr1}\\
y_{sgml} &&= 0,~\forall e_{ml}\notin SP_{sg}, \forall c_{sg} \in C.\label{constr:SP_constr2}
\end{IEEEeqnarray}
\end{subequations}
Here, $\mathcal{L}_{sg}$ represents the maximum flow for commodity $c_{sg}\in C$, defined by (\ref{constr:L_sg}). $\mathcal{L}_{in}$ represents the maximum influx among all the cells, defined by (\ref{constr:L_in}), where $\mathbf{I}_{v^{i}} = \left[ I_{y_{1}}, \cdots, I_{y_{F}}\right]$, where $y_{f} \in \left\{y_{sgml}\right\}, I_{y_{f}}\in \left\{0, 1\right\}$ is an indicator function that returns $1$ when $y_{f}$ has positive flow to an intermediate vertex $v^{i}\in V_{p}$. $\mathbf{y}$ is the vector of all flows. By minimizing both, the objective function penalizes the maximum congestion among all cells and disperses the flows related to one commodity. We weight $\mathcal{L}_{sg}$ and $\mathcal{L}_{in}$ with coefficients $\alpha$ and $\beta$, respectively, and set $\beta \gg \alpha$ to prioritize minimizing $\mathcal{L}_{in}$. (\ref{constr:mcf}) is the set of constraints for MCF formulation. The set of constraints (\ref{constr:SP_constr1}, \ref{constr:SP_constr2}) enforces the flows $y_{sgml}$ onto the shortest paths between the start and goal cells $SP_{sg}$ guaranteeing the solution optimality.


Despite the minimized influx to the intermediate cells, the  minimal influx MCF formulation leads to congestion in certain cells due to tight constraints on the shortest paths (robots' shortest paths may intersect at certain cells). To detour the robots optimally, ensuring that the number of robots that enter any intermediate cell $m$ remains below its influx limit $\theta_{m}$ and no unnecessary detour is introduced, we present a complete and optimal solver, MCF/OD, in Algo.~\ref{alg:MCF_optimal_detouring}. It maintains a conflict tree and resolves congestion iteratively. The function $congestionDetection(\mathcal{G}_{p}, P.solution, \boldsymbol{\theta})$ computes the influx for each cell in $\mathcal{G}_{p}$, given the flow solution $P.solution$ and returns the set of cells with their influx larger than the corresponding limit. The function $getAllConflict(P)$ returns the set of all commodities passing congested cells. 
\setlength{\textfloatsep}{0pt}
\begin{algorithm}[t]
\scriptsize
\caption{MCF/OD}
\label{alg:MCF_optimal_detouring}

\KwIn{a multi-commodity flow instance.}
Root.conflict\_counts$[o]$ = $1, \text{for } o = 1,\cdots, O$\\
Root.shortest\_paths = solve shortest path for each commodity in the $\mathcal{G}_{p}$\\ 
Root.cost = sum of the largest cost among all paths in each commodity\\
Root.solution = solve MCF with root shortest path constraints\\ 
insert Root to OPEN\\
Visited[Root.conflict\_counts] = True\\
\While{OPEN not empty} {
    $P \leftarrow$ node from OPEN with the lowest cost\\
    $CS \leftarrow congestionDetection(\mathcal{G}_{p}, P.solution, \boldsymbol{\theta})$\\
    \If{CS is empty} {
        \Return{P.solution}
    }
    $C \leftarrow $ $getAllConflict\left(P, CS\right)$.\\
    \For{commodity $c_{o}\in C$} {
        $A\leftarrow $ new node.\\
        $A$.conflict\_counts $\leftarrow P$.conflict\_counts + $\mathbf{I}_{o}$\\
        \If{Visited[A.conflict\_counts]} {
            Continue
        }
        $A$.shortest\_paths $\leftarrow P$.shortest\_paths\\
        $p_{k} \leftarrow generateKShortestPath(c_{o}, \mathcal{G}_{p},$ $A$.conflict\_counts$[o])$\\
        Update $A$.shortest\_paths[o] with $p_{k}$ \\
        \If{$p_{k} \leq w_{mcf} \cdot p_{min}$} {
            Update $A$.cost\\
            Update $A$.solution by solving MCF with updated shortest paths\\ 
            \If {$A$.cost $\leq \infty$} {
            Insert $A$ to OPEN
        }
        }
    }
}
\end{algorithm}
\begin{theorem}
MCF/OD is complete on a locally finite graph. 
\end{theorem}
\begin{proof}
The cost of a conflict tree node equals the sum of the costs of the longest routing (without cycles) in all commodities. For each expansion, $k$-th shortest paths will be added to the commodity,
which means the cost of the conflict tree is monotonically non-decreasing. For each pair of costs $X < Y$, the search will expand all nodes with cost $X$ before it expands the node with cost $Y$. As the graph is locally finite, there are a finite number of routing with the same cost for each commodity. Thus, expanding nodes with cost $X$ requires a finite number of iterations.

To include an arbitrary combination of $\hat{Z}$ unique edges 
of all commodities, the minimal cost of the conflict tree node is $Z$. $Z$ is finite, as the worst-case scenario is to include all the cell routing within the suboptimality bound. 
Since we are considering a graph with well-defined edge weights and a finite number of commodities, the worst cost is finite. 
For a finite cost $Z$, because the conflict tree node cost is monotonically non-decreasing and only a finite number of nodes with the same cost exists, we can find arbitrary combinations of $\hat{Z}$ unique edges in finite expansions. Thus, if a solution exists by including a combination of $\hat{Z}$ unique edges in the MCF, the algorithm can find it within finite expansions.
If all the unique edges have been added to the MCF solver and the optimization cannot find the solution that satisfies the user-defined influx limit, the problem is identified as unsolvable. 
\end{proof}
\begin{theorem}
MCF/OD is optimal. If a solution is found, it will have the lowest possible cost, i.e., the sum of the costs of the longest routing in all commodities will be minimized if a solution is found.
\end{theorem}
\begin{proof}
MCF/OD is a best-first search. In each expansion, the $k$-th shortest path for the selected commodity is inserted. Thus, the cost of a descendant node is monotonically non-decreasing. Therefore, if a solution is found, it is the optimal solution w.r.t. the cost. 
\end{proof}
MCF/OD can find the optimal detouring solution. However, the complexity is high due to solving an ILP in each expansion. 
To tackle many commodities in a large partition, we propose another efficient detour algorithm, one-shot MCF, which solves MCF once. One-shot MCF augments the shortest paths in \eqref{constr:SP_constr1}, \eqref{constr:SP_constr2} to include all the $w_{\mathrm{mcf}}$ bounded-suboptimal paths for each commodity. We employ the $k$-th shortest path routing algorithm to find all the candidate paths. The proposed One-shot MCF is complete as it includes all bounded-suboptimal paths for each commodity. While it does not optimize for the routing length for all commodities, it optimizes for minimum influx.
Intuitively, MCF/OD adds bounded-suboptimal paths iteratively to relax the constraints and terminates once all cell influx limits are satisfied. On the other hand, One-shot MCF adds all bounded-suboptimal paths at once and optimizes for the minimum influx, so it could result in unnecessary detour.

In each high-level planning iteration, we run both MCF/OD and One-shot MCF in parallel. If MCF/OD times out, we use the solution generated from One-shot MCF. The high-level replanning happens every $\delta_{h}$ time interval.




\section{Low-level Planner}
Within each cell, the low-level planner, or the cell planner,
computes collision-free paths that navigate robots to their local goals in an anytime fashion.
The cell planner can be divided into three steps: 1) local goal assignment
, 2) anytime MAPF/C generates discrete paths, and 3) cell-crossing protocol for non-stop transiting between cells. 

\subsection{Local Goal Assignment}
As the first step, local goal assignment aims to route robots to the closest local goals while spreading out robots optimally 
by solving the following ILP:
\begin{subequations}
\begin{IEEEeqnarray}{rCl'rCl'rCl}
\argmin_{\boldsymbol{\mathbf{A}}} \quad \sum_{ij} A_{ij} && \cdot D_{ij} + \alpha \sum_{j} u_{j} + \beta U  \\
\text{s.t.} \quad \sum_{j} A_{ij} && = 1, ~\forall i \\ 
U \geq u_{j} && \geq \sum_{i}A_{ij} - 1, ~\forall j,
\end{IEEEeqnarray}
\end{subequations}
where $A_{ij} \!\!\in\!\! \{0,1\}$ indicates if robot $r^{i}$ is assigned to the $j$-th local goal, denoted as $lg^{j}$. $D_{ij}$ is the Euclidean distance between $r^{i}$ and $lg^{j}$. 
Auxiliary variables $u_{j}$ in the objective function minimize the number of robots queueing at local goal $lg^{j}$, prioritizing  filling less congested local goals first. 
Auxiliary variable $U$ in the objective function minimizes the maximum number of robots waiting in queue among all the local goals. 
This leads to evenly routing robots to different local goals to reduce congestion. A local goal is occupied if assigned with at least one robot. Thus, the number of robots waiting in queue for a local goal $lg^{j}$ is $\left(\sum_{i} A_{ij} -1\right)$. 
\subsection{Anytime MAPF/C}
We adopt the SOTA anytime MAPF method, namely LNS~\cite{li2021anytime}, which iteratively improves the solution quality until a solution is needed, to facilitate real-time replanning.
We use ECBS~\cite{barer2014suboptimal} as the initial planner as it provides a bounded-suboptimal solution and prioritized planning with SIPP~\cite{phillips2011sipp} to rapidly replan for a subset of robots. 
We extend both ECBS and SIPP  using MAPF/C to account for the robot embodiment. 

For priority planning with SIPP
, we propose the following SIPP with generalized conflicts algorithm. 

\subsubsection{SIPP with generalized conflicts}
SIPP compresses the time dimension into sparse safety intervals to significantly reduce the search space. The SIPP configuration augments the position with its safety intervals. In the resulting configuration space, $A^{*}$ finds the shortest path for a robot.  Planned robots are considered moving obstacles and 
modify the safety interval of the traversed states. We propose SIPP with generalized conflicts (SIPP/C) with the following modifications and a different \textit{getSuccessors(s)} algorithm, where line $11 - 14$ (the highlighted part) differs from the original algorithm. 
In SIPP/C, the collision intervals, the complements of safety intervals, are added to vertices and edges as follows:

\noindent\textbf{SIPP/C vertex conflict:} for a robot at the vertex $u_{k}$ in a planned path at time step $k$, we add the collision interval $[k,k]$ to vertices and edges that intersect with the robot-robot collision model at $pos(u_{k})$, i.e., $\mathcal{R}_{\mathcal{R}}(pos(v)) \cap \mathcal{R}_{\mathcal{R}}(pos(u_k))$ and $\mathcal{R}^{*}_{\mathcal{R}}(e) \cap \mathcal{R}_{\mathcal{R}}(pos(u_k))$, $\forall v,e$. Note here, we use the \textit{swept} model $\mathcal{R}^{*}_{\mathcal{R}}(e)$ when the robot traverses an edge.

\noindent\textbf{SIPP/C edge conflict:} for a robot traversing an edge $e_{k}\!\!:=\!\!(u_{k}, u_{k+1})$ at time step $k$, we add the collision interval $[k, k]$ to vertices and edges that intersect with $\mathcal{R}^{*}_{\mathcal{R}}(e_{k})$, i.e., $\mathcal{R}_{\mathcal{R}}(pos(v)) \cap \mathcal{R}^{*}_{\mathcal{R}}(e_{k})$ and $\mathcal{R}^{*}_{\mathcal{R}}(e) \cap \mathcal{R}^{*}_{\mathcal{R}}(e_{k})$, $\forall v,e$.

\setlength{\textfloatsep}{0pt}
\begin{algorithm}[t]
\scriptsize
\caption{getSuccessors(s) for SIPP/C}
\label{alg:SIPP_with_MAPF/C}
$successors = \emptyset$\\
\For{each action e in E(s)} {
    $cfg$ = configuration of $e$ applied to $s$\\
    e\_time = time to execute action $e$\\
    start\_t = time($s$) + e\_time\\
    end\_t = interval($s$).end + e\_time\\
    \For{each safe interval i in cfg} {
        \If{i.start $>$ end\_t $\mid\mid$ i.end $<$ start\_t} {
            continue
        }
        $t$ = earliest arrival time at $cfg$ during interval $i$ with no collisions\\
        \tikzmk{A}\If{interval($e$).end $\geq t \!- \!e\_time ~\&\!\&$ 
        interval($e$).start $\leq$ interval($s$).end $~\&\!\&$ 
        $i$.start $\leq$ interval($e$).start + e\_time $\leq$ $i$.end
        } {
            t = max(interval($e$).start + m\_time, t)\\
        } 
        \Else {
            continue
        }\tikzmk{B}
        \boxit{myyellow}
        $s^{'}$ = state of configuration $cfg$ with interval $i$ and time $t$\\
        insert $s^{'}$ into $successors$
    }
}
\Return{successors}
\end{algorithm}

In Algo.~\ref{alg:SIPP_with_MAPF/C}, $E(s)$ is the action space at state $s$. 
For a discrete path $\mathcal{P}^{i}$, we assign a time $t_{k} = k\Delta t$ to each discrete time step and obtain the path $f^{i}$. $\Delta t$ is a user-defined value to satisfy the robot's dynamic constraints. The low-level replanning happens repeatedly with a $\delta_{l}$ time interval.

\subsection{Cell-crossing Protocol}
A robot would idle at a local goal if the path in its next cell is not yet computed. 
We propose a cell-crossing protocol that results in non-stop execution, even when traversing between cells. 
We buffer the hyperplane $H_{ml}$ by a distance $d_{e}$ towards cell $P_{m}$. All robots, currently in $P_{m}$ and exiting to $P_{l}$, compute paths for $P_{l}$ within the buffer. 
Buffering is achieved by changing the hyperplane offset to $\mathcal{H}^{'}_{a} = \mathcal{H}_{a} - \mathcal{H}_{\mathbf{n}}\cdot d_{e}$.
By enforcing the buffer distance $d_{e} \geq \delta_{l} \cdot V_{max}$, where $V_{max}$ is the maximum robot speed, 
the robot is guaranteed to have a plan in its next cell computed at least once 
before leaving its current cell.
A robot entering the buffer zone will then have a plan to exit its current cell and transition through its next cell. The robot, in the buffer zone, fixes its plan of the current cell to lock the local goal and expected arrival time to its next cell. Thus, when computing the plan for the next cell, the robot's expected start time and position will be pre-determined and independent of cell planning order.
The robot computes a plan for its next cell, then concatenates the fixed plan of its current cell and the plan in its next cell to form a complete transition plan. 
Fig.~\ref{fig:between_partition_planning} depicts our cell-crossing protocol.


\begin{figure}[t]
    \centering
    \includegraphics[width=0.16\textwidth]{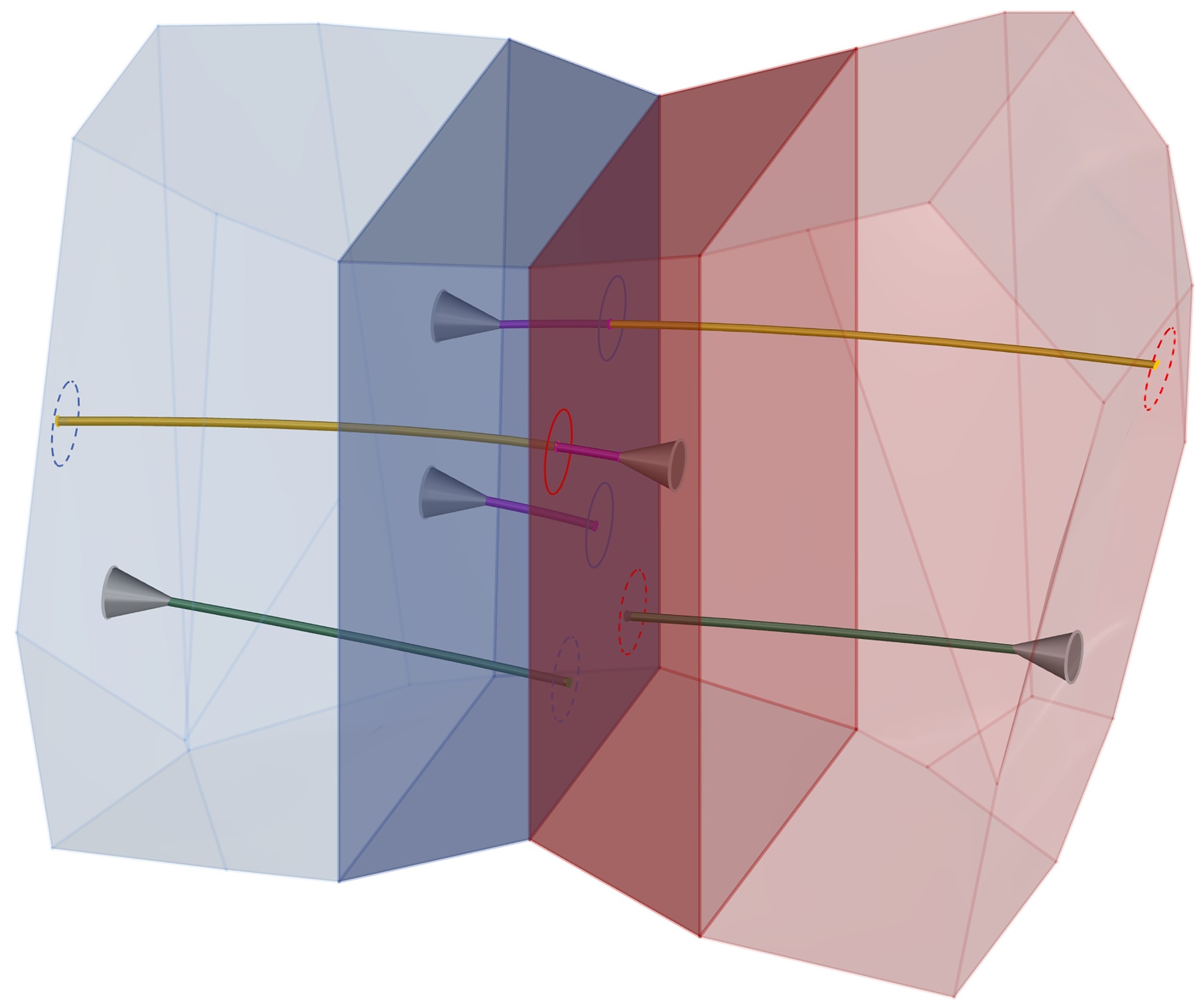}
    \caption{Robots (cones) fix their paths within the buffer zone and start replanning after crossing the hyperplane. 
    }
    \label{fig:between_partition_planning}
\end{figure}

\begin{table*}[t]
\vspace{+1em}
\centering
\scalebox{0.8}{
\begin{tabular}{c||c||c|c||c|c|c|c||c|c|c|c}
\multicolumn{1}{c||}{Instance} & \multicolumn{1}{c||}{ Q } & \multicolumn{2}{c||}{ Roadmap } & \multicolumn{4}{c||}{ High-level }  & \multicolumn{4}{c}{ Low-level } \\
&    & $\left|V\right|$ & $\left|E\right|$ & Method & $\bar{t}_{\mathrm{high}}(\unit{s})$ & $\bar{t}^{\mathrm{max}}_{\mathrm{high}}(\unit{s})$ & $\bar{N}_{\mathrm{max}}$ & Method & $\bar{T}$ & $\bar{t}_{\mathrm{low}}(\unit{s})$ & $\bar{t}^{\mathrm{max}}_{\mathrm{low}}(\unit{s})$ \\
\hline \hline
Demo32 & 8  & 1059 & 5890 & MCF & $0.06$ & 0.06 & 18.0 & ECBS(3.0)+PP & $26.88$ & $0.03\pm 0.01$ & $0.13\pm 0.06$\\
\hline
Circle74 & 1  & 1110 & 6728 & - & - & - & 29.0 & ECBS(2.0) & $\mathbf{29.00}$ & $12.31\pm 0.17$ & - \\
\hline 
Circle74 & 10 & 1515 & 8760 & Greedy & - & - & 42.4 & ECBS(2.0)+PP & $46.85$ & $0.04\pm 0.03$ & $0.40\pm 0.42$\\
\hline
Circle74 & 10   & 1515 & 8760 & MCF & $0.01$ & 0.04 & $\mathbf{23.9}$ & ECBS(2.0)+PP & $42.85$ & $\mathbf{0.02\pm 0.00}$ & $\mathbf{0.09\pm 0.02}$\\
\hline 
Circle142 & 1  & 2100 & 13220 & - & - & - & 32.0 & ECBS(4.5) & $\mathbf{36.00}$ & $16.32\pm 0.28$ & - \\
\hline
Circle142 & 25  & 2954 & 16383 & Greedy &  - & - & 73.6 & ECBS(4.5)+PP &  $121.60$ & $1.69\pm 0.82$ & $37.58\pm 23.07$ \\
\hline
Circle142 & 25  & 2954 & 16383 & MCF & $0.13$ & 0.35 & $\mathbf{27.1}$ & ECBS(4.5)+PP & $98.35$ & $\mathbf{0.03\pm 0.01}$ & $\mathbf{0.28\pm 0.14}$ \\
\end{tabular}
}
\vspace{-0.5em}
\caption{Influence of different parameters on different instances. The statistics are averaged across 20 trials.}
\label{table:quantitative}
\vspace{-2.5em}
\end{table*}


\section{Results and Discussion}
\label{sec:results}
We now demonstrate the system  in experiments on simulated and physical robots.
For large-scale simulated robot experiments, we create a confined 3D space 
with random obstacles uniformly generated on a disk of radius $10$\unit{m}. To validate the algorithm's scalability, we scale up the number of robots and the corresponding workspace size to maintain the robot density in different experiment instances. 
The  ``Circle74'' workspace is $20$\unit{m}$\times 20$\unit{m}$\times 8$\unit{m}. We generate 74 robots whose start states form a circle with a $10$\unit{m} radius at the height of $1$\unit{m} and are centered at the $x$-$y$ plane's origin, depicted in Fig.~\ref{fig:qualitative}a. 
In  ``Circle142'' we scale $x$ and $y$ dimensions by $\sqrt{N}$ to $27.7$\unit{m}$\times 27.7$\unit{m}$\times 8$\unit{m}. The robots start in concentric circles with $13.85$\unit{m} and $11.85$\unit{m} radii at $1$\unit{m} high, centered at the $x$-$y$ plane's origin, shown in Fig.~\ref{fig:qualitative}b. 
In  ``Demo32" we model the occupancy map of our cluttered lab environment. It is 
$12.55$\unit{m}$\times 7.63$\unit{m}$\times 2.8$\unit{m}. 
We run 32 Crazyflies with initial $x$-$y$ positions uniformly on an ellipse at $1$\unit{m} high, shown in Fig.~\ref{fig:qualitative}c.
The goal states are the antipodal points on the circle (or ellipse). We construct the roadmap using a 6-connected grid graph with an edge length of $1.6$\unit{m} for large-scale simulation and $0.7$\unit{m} for the lab environment. 

For simplicity, we set the same influx limit $\theta_{m} \!\!=\!\! 20$ for all cells in high-level planning and a suboptimal bound $w_{\mathrm{mcf}}\!\! =\!\! 2$ for both MCF/OD and One-shot MCF algorithms. We set the high-level planning time interval $\delta_{\mathrm{h}}\!\! =\!\! 5$\unit{s} for ``Circle74'',  $\delta_{\mathrm{h}}\!\! =\!\! 3$\unit{s} for ``Circle142'', and $\delta_{\mathrm{h}}\!\! =\!\! 10$\unit{s} for ``Demo32''. We set the low-level planning time interval $\delta_{\mathrm{l}}\!\! =\!\! 1$\unit{s}. For the LNS planner with random neighborhood selection, we use ECBS as the initial planner and prioritized planning with SIPP as the iterative planner, both planners we extend with MAPF/C. 
To account for downwash, we use an axis-aligned bounding box to represent the robot-robot collision model $\mathcal{R}_{\mathcal{R}}(\mathbf{0})$. 
In simulations, i.e., ``Circle74'' and ``Circle142'', we use the axis-aligned bounding box from $\left[-0.12\unit{m}, -0.12\unit{m}, -0.3\unit{m} \right]^{\top}$ to $\left[0.12\unit{m}, 0.12\unit{m}, 0.3\unit{m} \right]^{\top}$. 
Since ``Demo32'' is denser, we use the bounding box from $\left[-0.12\unit{m}, -0.12\unit{m}, -0.2\unit{m} \right]^{\top}$ to $\left[0.12\unit{m}, 0.12\unit{m}, 0.2\unit{m} \right]^{\top}$. We use the same shape representation for the robot-environment collision model $\mathcal{R}_{\mathcal{E}}(\mathbf{0})$. All experiments run on an Intel i7-11800H CPU computer. 

Fig.~\ref{fig:qualitative} depicts typical solutions of the proposed algorithm.
We summarize quantitative results in Table~\ref{table:quantitative}, where MCF refers to  MCF/OD and One-shot MCF running in parallel, as described in Sec.~\ref{sec:high_level_formulation}. 
Note that, as $\left| V\right|$ and $\left|E\right|$ suggest, the number of vertices and edges increases after partitioning as we add local goals and corresponding edges.
With a small computational overhead $\bar{t}_{\mathrm{high}}$, the proposed high-level planner effectively reduces the congestion among cells compared to both greedy and partitionless baseline approaches, by inspecting $\bar{N}_{\mathrm{max}}$. Here $\bar{t}_{\mathrm{high}}$ is the average high-level planning time and $\bar{N}_{\mathrm{max}}$ is the averaged maximum number of robots in a cell throughout the whole execution. 
By increasing the number of cells $Q$, the algorithm significantly reduces MAPF computation time. 
Specifically, for the average low-level planning time $\bar{t}_{\mathrm{low}}$, in instance ``Circle74", the proposed algorithm runs 616-times faster than the baseline method and 544-times faster in instance ``Circle142".   Both the average low-level replanning time $\bar{t}_{\mathrm{low}}$ and the averaged maximum low-level replanning time $\bar{t}^{\mathrm{max}}_{\mathrm{low}}$ are within the real-time regime for all instances.
Since we use an anytime algorithm, we only record its initial planning time in  Table~\ref{table:quantitative}, which can be directly compared to the baseline.
As we would expect, while running in real-time, the algorithm yields suboptimal solutions compared to the baseline MAPF, according to the average makespan $\bar{T}$. Because the partitioning invalidates the global optimality of MAPF algorithms, and the high-level planner detours robots among the partition, the planned paths become longer. 
\begin{figure}[t]
    \begin{subfigure}{0.155\textwidth}
       \centering {\footnotesize(a)}\includegraphics[height=.99\textwidth]{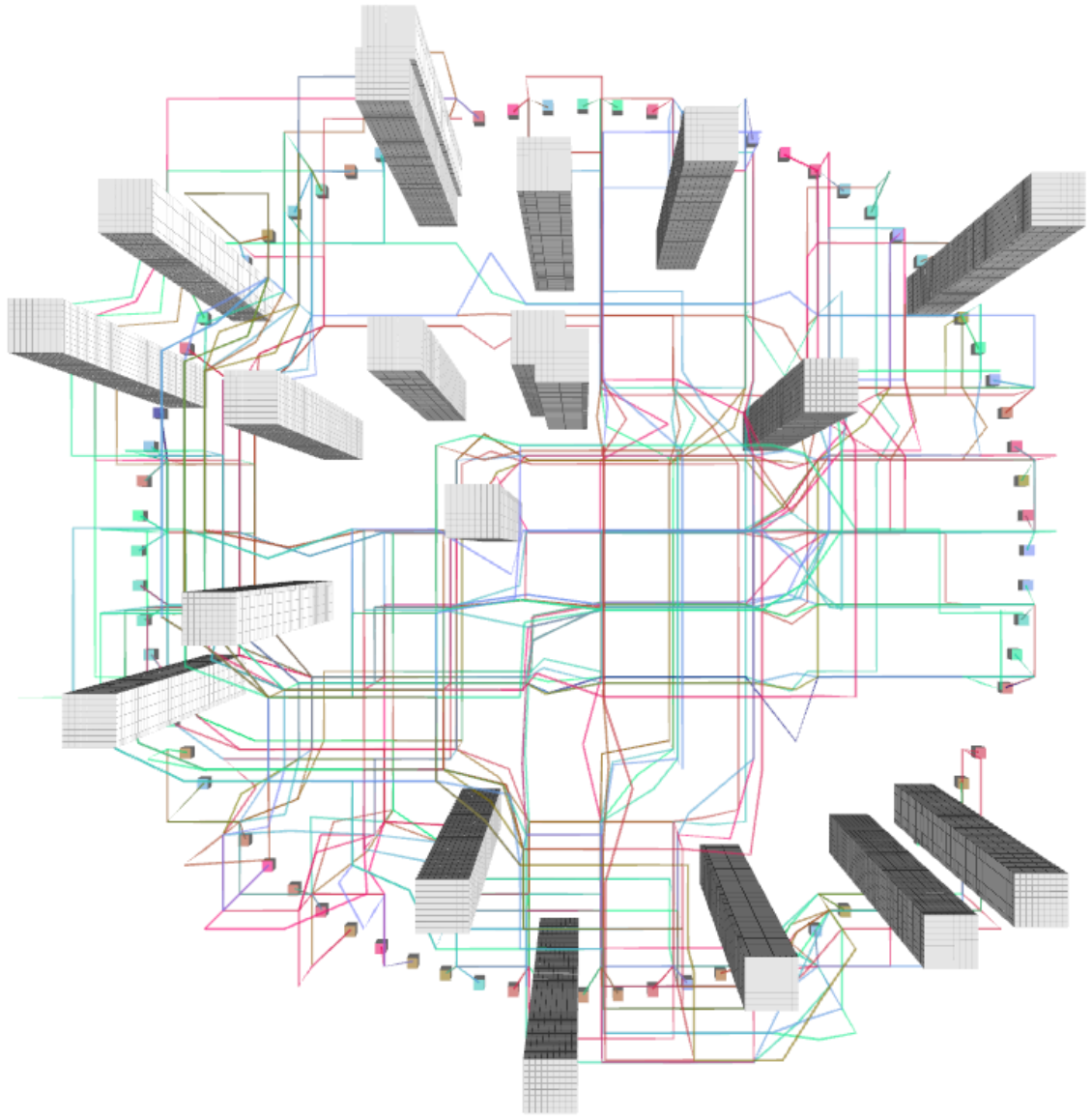}
    \end{subfigure}
    \begin{subfigure}{0.155\textwidth}
        \centering {\footnotesize(b)}\includegraphics[height=.99\textwidth]{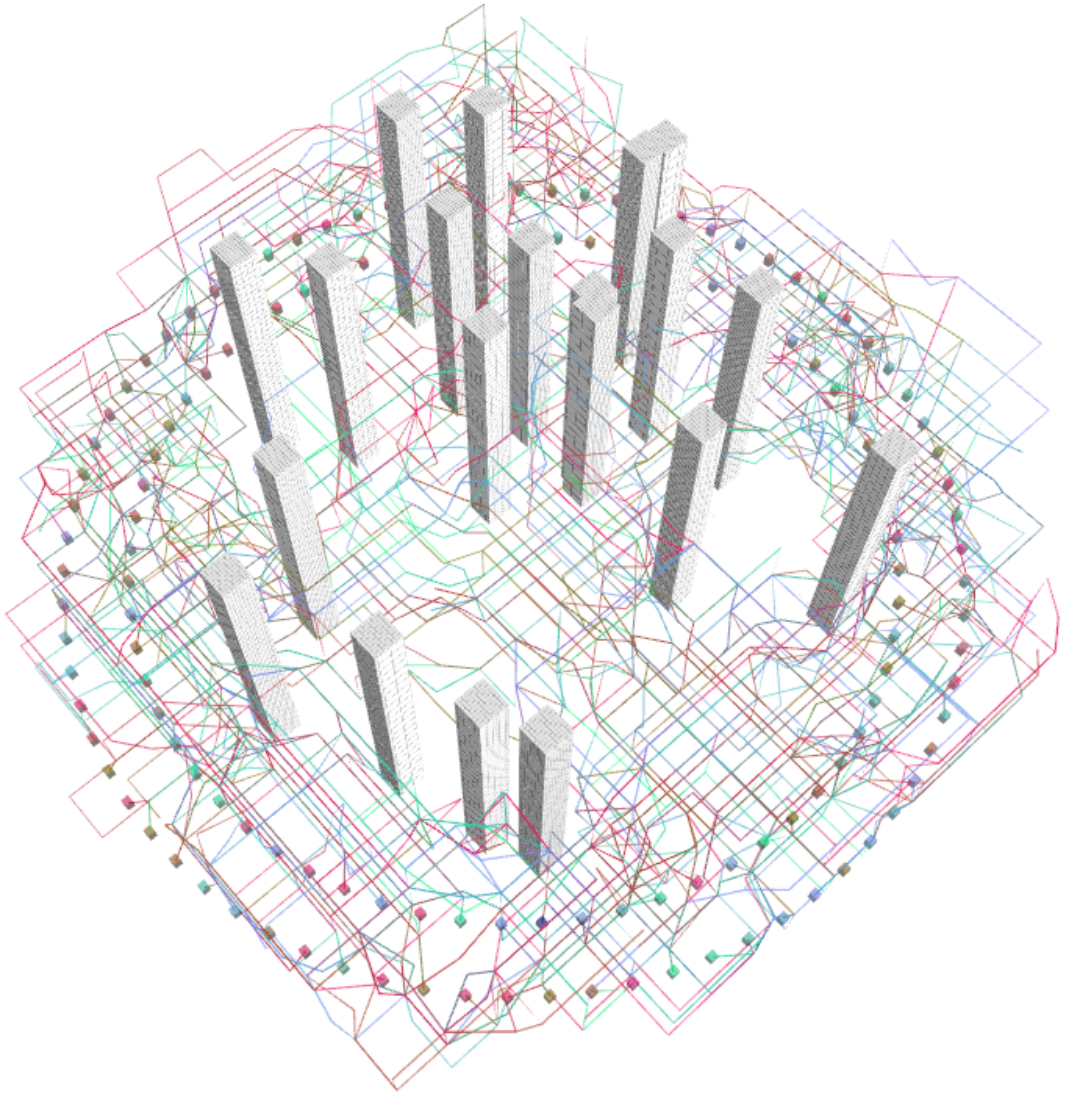}
    \end{subfigure}
    \begin{subfigure}{0.12\textwidth}
        \centering {\footnotesize(c)}\includegraphics[height=.99\textwidth]{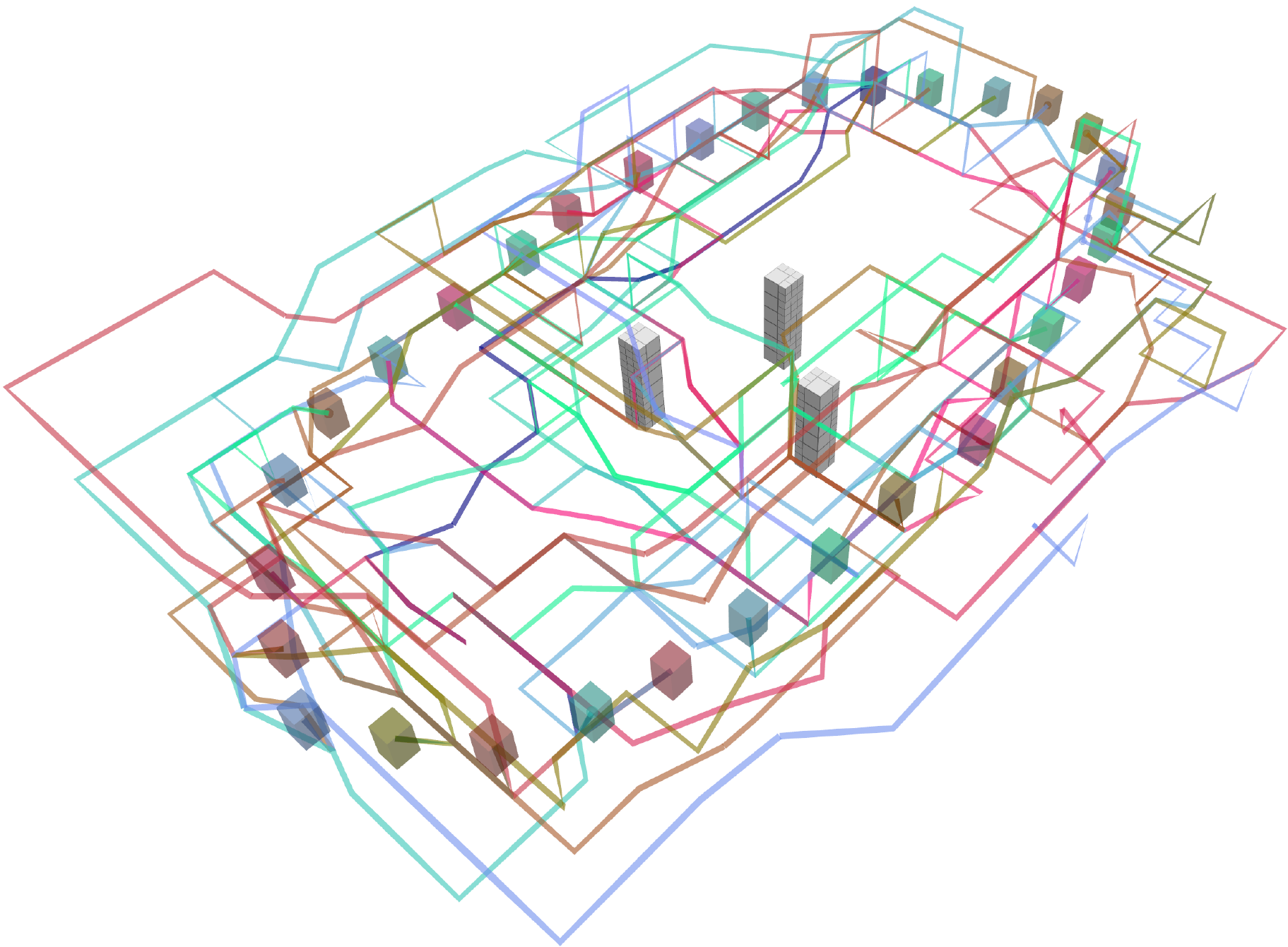}
    \end{subfigure}
\caption{Representative runs of our algorithm with (a) 74 robots, (b) 142 robots, and (c) 32 robots in cluttered environments.} 
\label{fig:qualitative}
\end{figure}
\subsection{Effectiveness of Partition in Low-level Planning}
Fig.~\ref{fig:qualitative_gp}a shows a quantitative evaluation of the low-level planning time and its solution makespan against the number of cells in ``Circle74". We report the median of low-level planning time for its robustness. The logarithm of low-level planning time decreases significantly as the number of cells increases. Thus, by dividing the workspace and parallelizing computation, the proposed algorithm significantly reduces the low-level planning time. The makespan increases at the beginning then plateaus. This phenomenon is expected. As the number of cells increases, at the beginning, the high-level planner becomes more effective in detouring the robots to reduce congestion. At a certain point, no further detour is required as congestion regulation is satisfied. Additionally, the partition breaks the global optimality of the MAPF planner, leading to a degeneracy in solution quality.
\begin{figure}[h]
    \centering
    \begin{minipage}{0.255\textwidth}
    \centering
    {\footnotesize(a)}
    \includegraphics[width=0.86\textwidth]{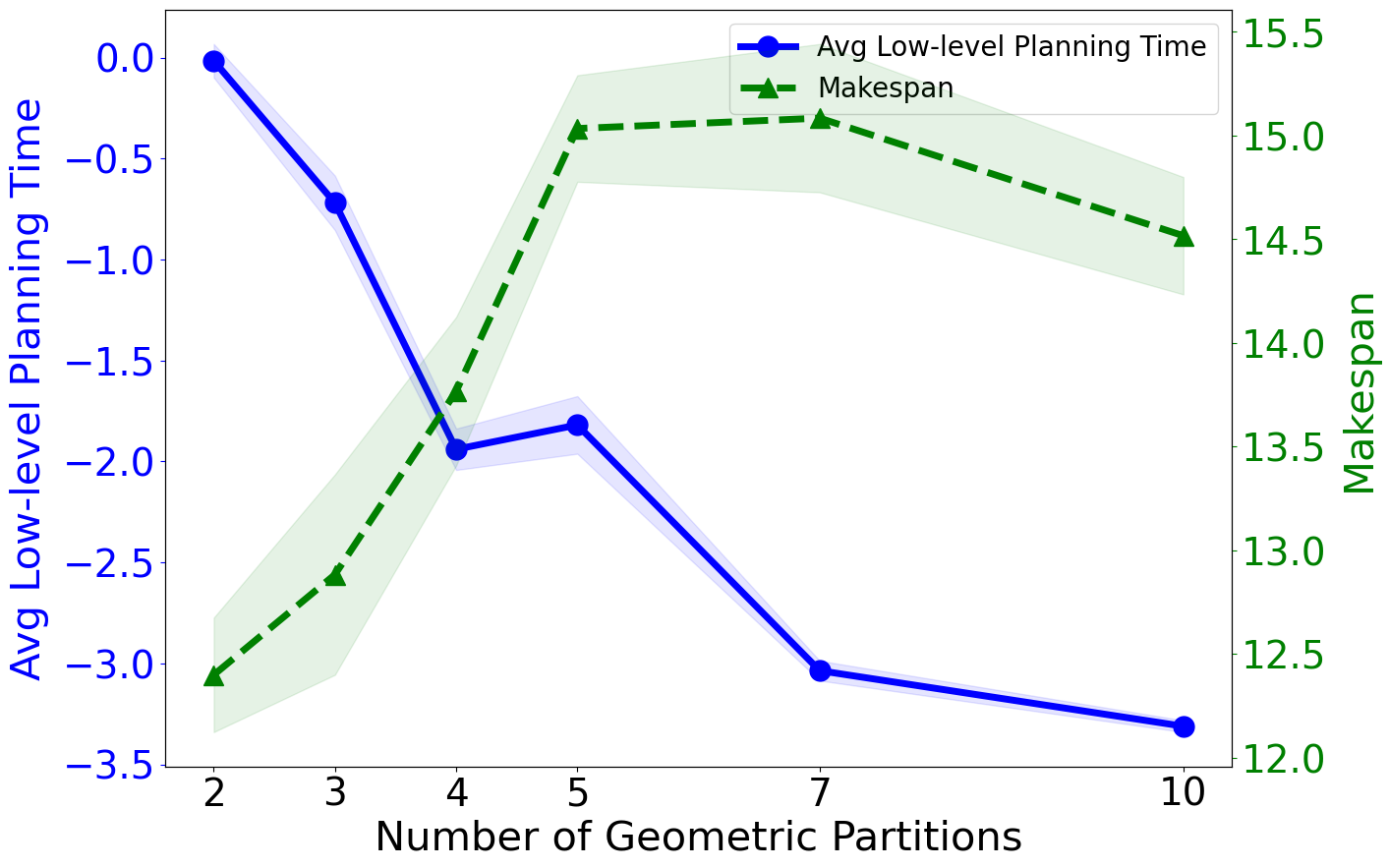}
    \end{minipage}
    \begin{minipage}{0.22\textwidth}
    \centering
    {\footnotesize(b)}
    \includegraphics[width=0.86\textwidth]{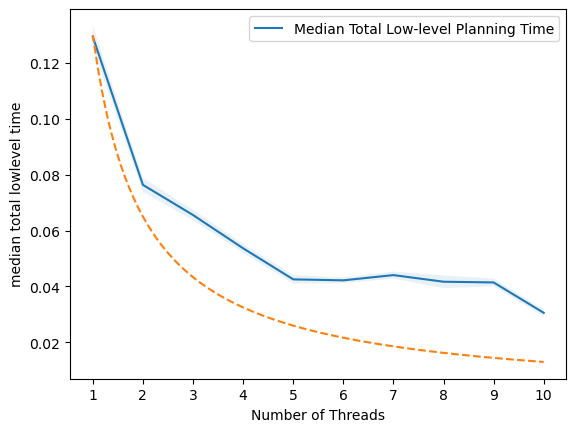}
    \end{minipage}
    \caption{95\% Confidence intervals of (a) log-median of low-level planning time and Makespan of our algorithm (b) median of total low-level planning time. The statistics are averaged across 20 trials.}
    \label{fig:qualitative_gp}
\end{figure}

Our algorithm runs MAPF in parallel for all cells; 
in Fig.~\ref{fig:qualitative_gp}b, we investigate the utilization of multi-threading. The orange dashed curve represents the theoretical lower bound of the total low-level planning time (assuming no overhead is introduced when allocating the computation to different threads). We observe that our method's computation time follows the same trend, while the gap between the theoretical lower bound and experimental computation time increases as we increase the number of threads, indicating an increase in parallelization overhead in more cells. 
\subsection{Effectiveness of High-level Planning}
To demonstrate the effectiveness of our MCF-based high-level planner, we compare its cell congestion to an egocentric greedy approach. The greedy planner outputs a single-robot-based shortest inter-cell routing without considering other robots' routing, and may lead to congestion in certain cells. Additionally, we compare the MCF-based approach with the baseline by imposing the partition onto the workspace. 

In Table~\ref{table:quantitative}, we report the average high-level computation time $\bar{t}_{\mathrm{high}}$ and the averaged maximum number of robots in a cell throughout the whole execution $\bar{N}_{\mathrm{max}}$. 
The complexity of a MAPF instance scales poorly with the number of robots. Thus,  $\bar{N}_{\mathrm{max}}$ is an insightful indicator of the computational hardness of a MAPF instance. The computation time for greedy high-level planning is instant, while the MCF-based methods have additional overhead that increases with the number of cells, since they are centralized. For all instances in this paper, MCF-based methods provide real-time solutions. 

The MCF-based methods reduce the congestion compared to greedy and baseline methods, according to $N_{\mathrm{max}}$. Although the advantage to the baseline is minor in the current experiment setup, the MCF-based approach can further reduce the congestion by relaxing the suboptimality bound, i.e., increase $w_{mcf}$. For a fair comparison of low-level planning time, both greedy and MCF-based utilize the same number of threads in computation. Compared to the greedy approach, the proposed MCF-based inter-cell routing brings cell planning computation time to real time while maintaining its solution quality, i.e., makespan. 
\subsection{Scalability of the Proposed Algorithm}
To validate the scalability of our algorithm, we run simulations on increasing numbers of robots (c.f., Table.~\ref{table:quantitative}). Note that our algorithm achieves real-time performance in all instances as we scale up, according to the averaged maximum high-level and low-level replanning time $\bar{t}^{\mathrm{max}}_{\mathrm{high}}$ and $\bar{t}^{\mathrm{max}}_{\mathrm{low}}$. Due to our hierarchical approach, it can further scale to larger teams with more cells and CPU threads. The computation bottleneck is the centralized high-level planning in larger scale problems, which we aim to address in future work.
\subsection{Physical Robots}
Fig.~\ref{fig:demo} shows a representative experiment with 32 physical Crazyflies in a cluttered environment (video link: \url{https://youtu.be/ftdWVpLkErs}).  We use a \textsc{Vicon} motion capture system to localize and CrazySwarm~\cite{preiss2017crazyswarm} to control the robots. In the experiment, robots are uniformly located on an ellipse, with their antipodal goals. Three column obstacles are placed within the ellipse. The experiment demonstrates that the proposed algorithm distributes  robots effectively through the workspace and achieves real-time replanning.

\section{Conclusion and Future Work}
We have introduced a hierarchical path planning algorithm for large-scale coordination tasks. Despite yielding a suboptimal solution, our algorithm significantly reduces the computation time and suits on-demand applications such as drone delivery. The framework achieves real-time operation by dividing the workspace into disjoint cells; within each, an anytime MAPF planner computes collision-free paths in parallel. Our high-level planner regulates congestion while guaranteeing routing quality. Additionally, 
our algorithm considers the robot embodiment, and we run experiments with the downwash-aware collision model of a quadrotor. 
We also devise a cell-crossing protocol, which guarantees non-stop execution even when transiting between cells. 

The proposed algorithm is designed for lifelong replanning. In our experiments, 
goals are chosen from a pre-determined set, however,  
it can be extended to a lifelong replanning as we request new goals online, and the algorithm runs conflict annotation to update the conflicts.

While the MCF-based high-level planner operates in real-time in our experiments with up to 142 robots in a 25-cell partition, the limits of its real-time operation are not well defined and depend on the number and density of robots and cells in the space.
Future work will explore distributed MCF~\cite{awerbuch2007greedy, awerbuch2012distributed} to increase scalability of the system with real-time, distributed operation regardless of the number of robots and cells. 
Furthermore, we aim to solve real time large scale motion planning, which respects the robot's kinodynamic constraints and plans in continuous space and time.



\bibliographystyle{IEEEtran}
\bibliography{IEEEabrv, refs}

\end{document}